%% file: COLT_arXiv.tex
\documentclass[12pt]{article}

\usepackage[colorlinks]{hyperref}            
\usepackage{color}
\usepackage{graphicx,subfigure,amsmath,amssymb,amsfonts,bm,epsfig,epsf,url,dsfont}
\usepackage{times}
\usepackage{bbm}      
\usepackage{booktabs}
\usepackage{cases}
\usepackage{fullpage}
\usepackage[small,bf]{caption}
\usepackage{natbib}
\usepackage[top=1in,bottom=1in,left=1in,right=1in]{geometry}

\include{Commands}

\newtheorem{theorem}{Theorem}
\newtheorem{lemma}{Lemma}
\newcommand{\BlackBox}{\rule{1.5ex}{1.5ex}}  
\newenvironment{proof}{\par\noindent{\bf Proof\ }}{\hfill\BlackBox\\[2mm]}

\begin{document}

\title{Most Correlated Arms Identification}
\author{
Che-Yu Liu, S{\'e}bastien Bubeck  \\
Department of Operations Research and Financial Engineering, \\
Princeton University \\
{\tt cheliu@princeton.edu} , {\tt sbubeck@princeton.edu}\\ 
}

\date{\today}

\maketitle

\begin{abstract}
We study the problem of finding the most mutually correlated arms among many arms. We show that adaptive arms sampling strategies can have significant advantages over the non-adaptive uniform sampling strategy. Our proposed algorithms rely on a novel correlation estimator. The use of this accurate estimator allows us to get improved results for a wide range of problem instances. 
\end{abstract}


\section{Introduction}

We define the {\em most correlated arms identification} problem as follows. Let $K$ be an integer and $X^t=(X_{1,t},\ldots, X_{K,t})^{\top}$, $t=1,2,\ldots$ be a sequence of independent, identically distributed Gaussian random vectors with zero mean and an {\em unknown} covariance matrix $\Sigma=(\sigma_{ij})_{i,j \in [K]}$. We use $\P_{\Sigma}$ to denote the corresponding probability measure  (on the natural probability space for this model) and make the following assumptions on the entries of $\Sigma$:
$ \forall i,j \in [K], \sigma_{ii}=1 \text{ and } \sigma_{ij} \geq 0 $.
Under these assumptions, $\Sigma$ can also be seen as the correlation matrix. Consider now an agent facing $K$ arms. At each time step $t=1,2,\ldots$, the agent selects a subset of arms $A_t \subset [K]$ and observes the corresponding values in $X^t$, denoted by $X^t_{A_t} = \{X_{i,t}, i\in A_t\}$. The agent is allowed to employ an {\em adaptive sampling strategy}, that is, the selection of $A_t$ may depend on the past observations $H_{t-1}=\{A_s, X^s_{A_s}\}_{s \in [t-1]}$. This is in contrast to the non-adaptive {\em uniform sampling} strategy, where $A_t$ is always chosen to be $[K]$. The task of the agent is to return as rapidly as possible a subset of $h$ most mutually correlated arms. More precisely, the agent is interested in finding
$$ S^{*} \in \argmax_{S\subset [K], |S|=h} \sum_{j,\ell \in S, j \neq \ell} \sigma_{j,\ell} .$$
Throughout this paper, we assume that there is an unique solution $S^{*}$ to the above problem and we call this unique solution the {\em optimal subset}. Furthermore an arm $i\in [K]$ is called {\em optimal} if it belongs to $S^{*}$ and is called {\em suboptimal} otherwise.  \\

The goal of this paper is to devise algorithms that output reliably $S^{*}$ while using as few {\em samples} as possible. Throughout this paper, the word {\em sample} is reserved for the realized value of a real-valued random variable. In comparison, the sampling outcome of a multi-dimensional random vector is called a {\em sample vector}. Following the tradition of the adaptive-exploration literature, we study our problem in two different settings.  

\noindent \textbf{Fixed-Budget:} In the fixed-budget setting, we are given a fixed budget of $n$ samples and are asked to output a subset  of arms $\hat{S}$ with size $h$ as soon as the sampling budget has been used up. The task here is to devise a sampling and decision strategy that achieves a small probability of returning a wrong subset, that is, one wants to minimize $\P_{\Sigma}(\hat{S} \neq S^{*})$. 

\noindent \textbf{Fixed-Confidence:} In the fixed-confidence setting, we are given a fixed confidence level $\delta > 0$ and there is no explicit budget limitation. We look for a sampling, stopping and decision strategy that returns $S^{*}$ with probability at least $1-\delta$ when it stops, regardless of the underlying correlation matrix $\Sigma$. The performance of the strategy is evaluated by the number of samples obtained before it terminates, either in expectation or with high probability. \\

In Section 4 and 5, we present two algorithms SR-C and SE-C that are designed for the above two settings respectively. The main innovation in our algorithms is the use of an accurate correlation estimator, called difference-based correlation estimator. In Section 2, we define and analyze this new estimator and argue that it can substantiate in an optimal way the intuition that the estimation task becomes easier when the correlation is close to $1$, a feature that the classical correlation estimator can not attain. Inspired by the statistical properties of the difference-based correlation estimator, we define the suboptimality ratio of $B$ with respect to $A$ (for the correlation matrix $\Sigma$)  as 
$$\mathcal{D}_{\Sigma}(A,B)= \frac{\sum_{(j,\ell) \in B^2 \backslash (A \cap B)^2, j \neq \ell} (1-\sigma_{j,\ell}) }{ \sum_{(j,\ell) \in A^2 \backslash (A \cap B)^2, j \neq \ell} (1-\sigma_{j,\ell})}$$
for two subsets of arms $A$ and $B$ with size $h$ and with the convention that $\frac{0}{0}=1$.
This ratio is a measure of how uncorrelated the arms in $B$ are compared to the arms in $A$.  Note that \\ $\mathcal{D}_{\Sigma}(A,B) \geq 1$ if only if $\sum_{j,\ell \in A , j \neq \ell} \sigma_{j,\ell} \geq \sum_{j,\ell \in B , j \neq \ell} \sigma_{j,\ell}$. Now for an arm $i \in [K]$, denote by $R_{i,\Sigma}$ the suboptimality ratio of the arm $i$ with respect to the optimal set $S^{*}$, defined as 
$$R_{i,\Sigma}= \min_{i \in B \subset [K], |B|=h} \mathcal{D}_{\Sigma}(S^{*}, B).$$ 
It is clear that $R_{i,\Sigma} = 1$ if arm $i$ is optimal and $R_{i,\Sigma} > 1$ otherwise. Therefore, the order statistics of $R_{i,\Sigma}$ satisfies 
$$1=R_{(1),\Sigma}=\ldots = R_{(h),\Sigma} < R_{(h+1),\Sigma} \leq \ldots \leq R_{(K),\Sigma}.$$ 
These suboptimality ratios determine the number of times the suboptimal arms need to be drawn in our algorithms. More precisely, define the $\alpha$ function as $\alpha(\theta) = \frac{1}{2}(\log \theta-1+\frac{1}{\theta})$, for $\theta \geq 1$. We consider in Section 3 the non-adaptive sampling setting where we draw each arm an equal number of times and show that we need $\tilde{ {\Theta}}\left(\frac{K}{\alpha(R_{(h+1),\Sigma})}+K \right)$ samples\footnote{The notation $\tilde{{\Theta}}, \tilde{\mathcal{O}}$ hides a logarithmic factor.} to reliably identify $S^{*}$ by providing a matching upper and lower bound. Furthermore, we prove in Section 4 and 5 that SR-C and SE-C need at most $\tilde{ \mathcal{O}}(\mathrm{H_C}+K)$ samples to find $S^{*}$ where $\mathrm{H_C}$ is defined as 
$$\mathrm{H_C} = \frac{h}{\alpha \left(R_{(h+1),\Sigma}\right)} + \sum_{i=h+1}^{K} \frac{1}{\alpha \left(R_{(i),\Sigma}\right)}.$$
Note that $\mathrm{H_C}$ can be much smaller than $\frac{K}{\alpha(R_{(h+1),\Sigma})}$ for a wide range of $\Sigma$, especially when the entries of $\Sigma$ are inhomogeneous. This shows that adaptive sampling strategies could have significant advantages over non-adaptive uniform sampling. We close the paper with a discussion in Section 6 of possible improvements and challenges. \\

\noindent \textbf{Remark:} Note that the function $\alpha: [1, +\infty) \rightarrow \R$ defined above is a positive, strictly increasing function. Moreover, $\alpha(\theta) = \Theta\left((1-\theta)^2\right)$ when $\theta \rightarrow 1$ and $\alpha(\theta) = \Theta\left(\log(\theta)\right)$ when $\theta \rightarrow +\infty$. It follows from these facts that for any $q > 0$, there exists two constants $c_1(q),c_2(q) >0$ that only depend on $q$ such that $c_1(q)  \alpha(\theta) \leq  \alpha(\theta^q) \leq c_2(q)  \alpha(\theta)  $ for any $\theta \geq 1$. 

\subsection*{Related Work}
The problem we study in this paper is similar in spirit to the {\em best arm identification problem} in multi-armed bandits. In the latter problem, an agent repeatedly selects an arm and observes a sample reward drawn from the arm's reward distribution, and then he is asked to return the arm with the highest mean reward.  The Successive Elimination algorithm  in \cite{EMM06} was shown to find the single best arm $i^{*}$ with high probability with $\tilde{\mathcal{O}} \left( \sum_{i \neq i^{*}} \frac{1}{\Delta_i^2} \right)$ samples where $\Delta_i$ is the gap between the mean reward of the best arm and that of a suboptimal arm $i$. \cite{BMS09} and \cite{ABM10} study the same problem under the fixed-budget setting, in particular the Successive Rejects procedure of \cite{ABM10} was shown to require essentially as many samples as Successive Elimination. The algorithms that we propose in Sections 4 and 5 for most correlated arms identification are inspired by these algorithms for best arm identification. In particular they share the same high-level idea that the more suboptimal an arm is with respect to the best arms, the less samples we need to distinguish it from the best arms. In the fixed-budget setting, this is done by distributing the sample budget to the arms in an adaptive way based on their correlation estimates. For the fixed-confidence setting we build confidence intervals on correlations among arms and uniformly sample all the arms until we have enough confidence to identify and exclude some suboptimal arms. 

Our work is different from \cite{CBL12}, \cite{CBL12b} and \cite{CLS13} where the task is the detection of the presence of a sparse and correlated subset of components from samples of a high-dimensional Gaussian distribution. Their work focuses on determining whether there is a correlated subset while the task in this paper is to accurately identify the subset with largest mutual correlations. Perhaps more importantly, on the contrary to these work, our algorithms adapt to the potential heterogeneity in the correlation matrix $\Sigma$.

\section{Correlation Estimators} 
In this section, we describe a new correlation estimator, called difference-based correlation estimator. It is intuitive that the task of correlation estimation becomes easier when the correlation is close to $1$. We are going to see that the classical correlation estimator fails to capture this intuition while the difference-based correlation estimator can quantify it in an optimal way.  The use of this novel estimator will allow us to largely reduce the sample complexity of our algorithms for a wide range of correlation matrices $\Sigma$. 

\subsection{Classical Correlation Estimator}
For $j \neq \ell \in [K]$, consider the classical correlation estimator defined as
\begin{eqnarray*}
\tilde{\sigma}_{j,\ell,t} &=& \frac{1}{t}\sum_{s=1}^{t} X_{j,s}X_{l,s} = \frac{1}{t}\left(\sum_{s=1}^{t}\frac{(X_{j,s}+X_{l,s})^2}{4}- \sum_{s=1}^{t}\frac{(X_{j,s}-X_{l,s})^2}{4} \right) \\
 &=&  \sigma_{j\ell}+ \left(  
 \frac{1+ \sigma_{j\ell}}{2} \left( \frac{1}{t}\sum_{s=1}^{t}\frac{(X_{j,s}+X_{l,s})^2}{2(1+ \sigma_{j\ell})}-1\right)  
 -  \frac{1- \sigma_{j\ell}}{2} \left( \frac{1}{t}\sum_{s=1}^{t}\frac{(X_{j,s}-X_{l,s})^2}{2(1- \sigma_{j\ell})}-1\right)   \right) .
\end{eqnarray*} 
By observing that $\sum_{s=1}^{t}\frac{(X_{j,s}+X_{l,s})^2}{2(1+ \sigma_{j\ell})}$ and $\sum_{s=1}^{t}\frac{(X_{j,s}-X_{l,s})^2}{2(1- \sigma_{j\ell})}$ are independent, chi-square distributed with $t$ degrees of freedom, we see that the magnitude of the fluctuation of $\tilde{\sigma}_{j,\ell,t} $ around  $\sigma_{j\ell}$ only has a very mild dependency on $\sigma_{j\ell}$. Its estimation accuracy improves only slightly even when $\sigma_{j\ell} = 1$.  

\subsection{Difference-Based Correlation Estimator}
Our main idea for getting a good correlation estimator is to draw inspiration from the likelihood ratio tests. Assume that $K=2$ and consider the following testing problem. 
$$ (T) \hspace{0.3cm} \begin{cases}
H_0 : \Sigma = \Sigma_0  \\
H_1 : \Sigma = \Sigma_1 
\end{cases}  $$
where $\Sigma_0 = \bigl( \begin{smallmatrix} 1 & \rho_0 \\ \rho_0 & 1 \end{smallmatrix} \bigr)$ and $\Sigma_1 = \bigl( \begin{smallmatrix} 1 & \rho_1 \\ \rho_1 & 1 \end{smallmatrix} \bigr)$
with $1 > \rho_0 > \rho_1 \geq 0$. Since the likelihood ratio test is an optimal test, it must be able to distinguish $H_0$ and $H_1$ with high accuracy in the case when $\rho_0$ is close to $1$ while $\rho_1$ is bounded away from $1$. This suggests that the study of the likelihood ratio test may be helpful in constructing a good correlation estimator. Denote by $f_{Q}$  the probability density function of a probability distribution $Q$. It is easy to verify that the likelihood ratio statistics of the testing problem $(T)$ can be written as 
$$ \Lambda((X^s)_{s=1,\ldots,t}) = \frac{f_{\mathcal{N}(0,(1+\rho_0)I_t)}\left( \left( \frac{X_{1,s}+X_{2,s}}{\sqrt{2}} \right)_{s=1,\ldots, t} \right) }{f_{\mathcal{N}(0,(1+\rho_1)I_t)}\left( \left( \frac{X_{1,s}+X_{2,s}}{\sqrt{2}} \right)_{s=1,\ldots, t} \right)}  
\cdot \frac{f_{\mathcal{N}(0,(1-\rho_0)I_t)}\left( \left( \frac{X_{1,s}-X_{2,s}}{\sqrt{2}} \right)_{s=1,\ldots, t} \right)}{f_{\mathcal{N}(0,(1-\rho_1)I_t)}\left( \left( \frac{X_{1,s}-X_{2,s}}{\sqrt{2}} \right)_{s=1,\ldots, t} \right)}. $$
The two fractions above can be seen as the likelihood ratio statistics of the following two testing problems.
$$ (T_1) \hspace{0.3cm} \begin{cases}
H_0 :  \frac{X_{1,s}+X_{2,s}}{\sqrt{2}} \sim  \mathcal{N}(0,1+\rho_0)  \\
H_1 : \frac{X_{1,s}+X_{2,s}}{\sqrt{2}} \sim  \mathcal{N}(0,1+\rho_1) 
\end{cases}, \hspace{1cm}
(T_2) \hspace{0.3cm} \begin{cases}
H_0 :  \frac{X_{1,s}-X_{2,s}}{\sqrt{2}} \sim  \mathcal{N}(0,1-\rho_0)  \\
H_1 : \frac{X_{1,s}-X_{2,s}}{\sqrt{2}} \sim  \mathcal{N}(0,1-\rho_1) 
\end{cases} .
$$
Now observe that the testing problem $(T_2)$ is always easier than the problem $(T_1)$ since we always have $\frac{1+\rho_0}{1+\rho_1} \leq \frac{1-\rho_1}{1-\rho_0}$. This suggests that the hardness of the problem stays roughly the same if we replace the original problem $(T)$ by problem $(T_2)$. Moreover, for problem $(T_2)$, it is natural to consider tests based on the test statistics $\frac{1}{t}\sum_{s=1}^{t}\frac{(X_{1,s}-X_{2,s})^2}{2}$. This leads us to  define the difference-based correlation estimator for $\sigma_{j\ell}$ with $j \neq \ell \in [K]$ as 
$$\hat{\sigma}_{j,\ell,t}=1-\frac{1}{t}\sum_{s=1}^{t}\frac{(X_{j,s}-X_{l,s})^2}{2}.$$ 
To analyze its statistical properties, we further define the random variable $Y_{j,\ell,t}$ as 
$$Y_{j,\ell,t}=\frac{1}{1-\sigma_{j\ell}}\sum_{s=1}^{t}\frac{(X_{j,s}-X_{l,s})^2}{2}$$
with the convention that $\frac{0}{0}=1$. Then $Y_{j,\ell,t}$ follows a chi-square distribution with $t$ degrees of freedom and we have the following relations
$$ 1-\hat{\sigma}_{j,\ell,t} = (1-\sigma_{j\ell})\frac{Y_{j,\ell,t}}{t}  \text{ and } 
\hat{\sigma}_{j,\ell,t} = \sigma_{j,\ell} - (1-\sigma_{j\ell})(\frac{Y_{j,\ell,t}}{t}-1). $$ 
For a fixed number of samples, the deviation of the difference-based correlation estimator from the true correlation is proportional to $1-\sigma_{j\ell}$. In other words, the accuracy of estimation increases when $\sigma_{j\ell}$ approaches $1$.

\subsection{Optimality of the Difference-Based Correlation Estimator }

To illustrate the strength of the difference-based correlation estimator, we return to the testing problem $(T)$.  The testing accuracy of a test $\phi$ is measured by its maximal risk, defined as 
$$ \mathcal{R}_t(\phi)= \max\left(\P_{\Sigma_0} (\phi=1), \P_{\Sigma_1} (\phi=0) \right).$$  
Now, define $R=\frac{1-\rho_1}{1-\rho_0} > 1$ and consider the test $\phi^{*}$ which outputs $0$ if $1-\hat{\sigma}_{1,2,t} \leq (1-\rho_0)R^{\frac{1}{2}}$ and outputs $1$ otherwise. Using Lemma 5 in the Appendix, one has the following upper bound on the maximal risk of $\phi^{*}$,
$$\mathcal{R}_t(\phi^{*}) \leq  \exp\left(-t\cdot \alpha(R^{\frac{1}{2}})\right) \leq  \exp\left(-c_1(1/2) \cdot t\cdot \alpha(R)\right). $$ 
On the other hand, using Lemma 6 and Lemma 7 in the Appendix, we know that the maximum risk for any test $\phi$ is lower bounded as follows.
\begin{eqnarray*}
\mathcal{R}_t(\phi)  & \geq &  \frac{1}{4}\exp\left(-  \KL\left( \mathcal{N}(0, \Sigma_0)^{\otimes t}, \mathcal{N}(0, \Sigma_1)^{\otimes t} \right) \right)  \\
    & =& \frac{1}{4}\exp\left(- t \cdot \KL\left( \mathcal{N}(0, \Sigma_0), \mathcal{N}(0, \Sigma_1)\right) \right) 
    \geq  \frac{1}{4} \exp\left(- c t \cdot \alpha(R) \right)
\end{eqnarray*}
where $c$ is an universal constant. We conclude that the test $\phi^{*}$ using the difference-based correlation estimator is an optimal test in the sense that it needs $\mathcal{O}(\frac{1}{\alpha(R)})$ samples to reliably identify the correct $\Sigma$ and no test can do better.

\section{Non-Adaptive Sampling Setting}
In this section, we study the most correlated arms identification problem under the classical non-adaptive setting where we have access to $m$ full sample vectors $(X^t)_{t=1,\ldots, m }$ (this corresponds to $n=Km$ samples of arms). We show that a naive decision policy based on the difference-based correlation estimator needs $n = \tilde{ \mathcal{O}}\left(\frac{K}{\alpha(R_{(h+1),\Sigma})}+K \right)$ samples to find $S^{*}$. We also provide a matching lower bound (up to a logarithmic factor on $K$) on problem instances with a certain correlation matrix structure.

\subsection{Upper Bound}
Consider the decision policy that simply outputs the empirically optimal subset $\hat{S}$, as described in Figure 1. 
\begin{figure}[h!]
\bookbox{ 
\textbf{Input:} $m$ full samples $(X^t)_{t=1,\ldots, m }$.  \\
\textbf{Output:} $ \hat{S} = \arg\max_{S\subset [K], |S|=h} \sum_{j,\ell \in S, j \neq \ell} \hat{\sigma}_{j,\ell, m}$
} 
\caption{
A naive decision policy under the non-adaptive sampling setting.}
\end{figure}

\begin{theorem}
The probability of error of the decision policy in Figure 1 is bounded as 
$$ \P_{\Sigma} (\hat{S} \neq S^{*}) \leq K(K-1) \exp\left(- cm \cdot \alpha\left(R_{(h+1),\Sigma} \right) \right) $$ 
where $c$ is an universal constant.
\end{theorem}
\begin{proof}
First, observe that if $\hat{S} \neq S^{*}$, then  
$$\left( \sum_{(j,\ell) \in \hat{S}^2 \backslash (S^{*} \cap \hat{S})^2, j \neq \ell} (1-\sigma_{j,\ell})\right) \,\geq  \,  R_{(h+1),\Sigma} \left( \sum_{(j,\ell) \in S^{*2} \backslash (S^{*} \cap \hat{S})^2, j \neq \ell} (1-\sigma_{j,\ell}) \right) $$ 
However, by definition of $\tilde{S}$, one has 
$$ \left( \sum_{(j,\ell) \in \hat{S}^2 \backslash (S^{*} \cap \hat{S})^2, j \neq \ell} (1-\hat{\sigma}_{j,\ell,m})\right) \,\leq  \, \left( \sum_{(j,\ell) \in S^{*2} \backslash (S^{*} \cap \hat{S})^2, j \neq \ell} (1-\hat{\sigma}_{j,\ell,m}) \right) $$ 
Therefore, if $\hat{S} \neq S^{*}$, then one of the following two events must hold true
$$\{ \exists j \neq \ell \text{ such that } (1-\hat{\sigma}_{j,\ell,m}) \geq R_{(h+1),\Sigma}^{\frac{1}{2}}  (1-\sigma_{j,\ell}) \}$$ 
$$\{ \exists j \neq \ell \text{ such that } (1-\hat{\sigma}_{j,\ell,m}) \leq R_{(h+1),\Sigma}^{-\frac{1}{2}}  (1-\sigma_{j,\ell}) \}.$$
Therefore, using Lemma 5, one gets
$$ \P_{\Sigma} (\hat{S} \neq S^{*}) \leq K(K-1) \exp\left(- m \cdot \alpha\left(R_{(h+1),\Sigma}^{\frac{1}{2}} \right) \right)
\leq K(K-1) \exp\left(- m \cdot c_1(\frac{1}{2}) \alpha\left(R_{(h+1),\Sigma} \right) \right) $$
which completes the proof.
\end{proof}

\subsection{Lower Bound} \label{sec:nonadaptivelowerbound}
Let $1 > \rho_h > \rho_{h+1} \geq \ldots \geq \rho_{K}$ and consider the correlation matrix $\Sigma$ defined by
$$ \sigma_{j \ell} = \begin{cases} 
1,  & \mbox{if } j = \ell \mbox{ or } j,\ell < h \\
\rho_j \rho_{\ell} , & \mbox{if } j \neq \ell \mbox{ and } j, \ell \geq h \\
 \rho_{\ell} , & \mbox{if } j  < h  \leq \ell \\
 \rho_{j} , & \mbox{if } j  \geq h  > \ell 
\end{cases} $$ 
It is easy to verify that the optimal subset $S^{*}$ is $[h]$ and for $i \geq h+1$, $R_{(i), \Sigma} = R_{i,\Sigma} = \frac{1-\rho_i}{1-\rho_h}$. Now let $\rho_{h+1}^{'} = 1 - \frac{(1-\rho_h)^2}{1-\rho_{h+1}}$ and define another correlation matrix $\Sigma^{'}$ which is constructed from $\Sigma$ by replacing $\rho_{h+1}$ by $\rho_{h+1}^{'}$ in the definition of $\Sigma$. By observing that $\rho_{h+1}^{'} > \rho_h$, we see that the optimal subset for $\Sigma^{'}$ is $S^{*}_{\Sigma^{'}} = [h-1] \cup \{h+1\}$ and
$$ R_{(h+1), \Sigma^{'}} = R_{h,\Sigma^{'}} = \frac{1-\rho_h}{1-\rho_{h+1}^{'}},   R_{(i), \Sigma} = R_{i,\Sigma} = \frac{1-\rho_i}{1-\rho_{h+1}^{'}}, \text{ for } i \geq h+2 .$$ 
Note that $R_{(i),\Sigma^{'}} \geq R_{(i),\Sigma} $ for any $i \geq h+1$. So the identification problem with $\Sigma^{'}$ is always easier than that with $\Sigma$. The following theorem shows that for any decision policy with uniform sampling, its probability of error is at least $ \Omega \left( \exp\left(- cm \cdot \alpha(R_{(h+1), \Sigma}) \right) \right) $ under one of the two problem instances with underlying correlation matrix $\Sigma$ and $\Sigma^{'}$.    

\begin{theorem}
Let $\phi$ be a decision policy that outputs the subset $\phi((X^t)_{t=1,\ldots, m })$ when $(X^t)_{t=1,\ldots, m }$ is observed. Then 
$$ \max\left( \P_{\Sigma} (\phi \neq S^{*}), \P_{\Sigma^{'}} ( \phi \neq S^{*}_{\Sigma^{'}} )  \right)
\geq  \frac{1}{4} \exp\left(- c m \cdot \alpha(R_{(h+1), \Sigma}) \right)
\geq  \frac{1}{4} \exp\left(- c m \cdot \alpha(R_{(h+1), \Sigma^{'}}) \right) .  $$
where $c$ is an universal constant.
\end{theorem}
\begin{proof}
The first inequality can be easily proved by using Lemma 6 and Lemma 8 in Appendix. The second inequality is trivial. 
\end{proof}

\section{Fixed-Budget Setting: Successive Rejects for Correlation}
In this section, we present and analyze a new algorithm, SR-C (Successive Rejects for Correlation). See Figure 2 for its complete description. SR-C proceeds by rounds. At the end of each round, it computes an estimate for the suboptimality  of each arm and rejects an arm that seems least likely to be an optimal arm. And during the next round, it samples the subset of existing arms a certain number of times. SR-C stops when there are only $h$ arms left and outputs the set of remaining arms. It is easy to verify that SR-C does not exceed the sample budget $n$ (the total number of samples it uses is $n_1 + \cdots + n_{K-h-1} + (h+1)n_{K-h}$).  The following theorem shows that SR-C needs at most $\tilde{\mathcal{O}}\left( \mathrm{H_C} +K \right)$ samples to find the optimal subset. More precisely, when the sample budget $n$ is of order $\mathcal{O}(\mathrm{H_C} \log(K)^2)$, we have a non-trivial upper bound on the probability of error. 

\begin{figure}[h!]
\bookbox{ 
\textbf{Input:} sample budget $n$.  \\
Let $S_1=[K], \overline{\log(\frac{K}{h})}=1+\sum_{i=h+1}^K \frac{1}{i} ,n_0=0$ and for $k=1,\ldots,K-h$, define
$$n_k=\lceil \frac{n-K-1}{\overline{\log(\frac{K}{h})}(K+1-k)} \rceil$$  
For each round $k=1,\ldots,K-h$, do \\
(1) For $t= n_{k-1}+1,\ldots, n_k$, choose $A_t = S_k$ and sample $A_t$ from $X^t$.\\
(2) For each $i \in S_k$, compute $U_{i,k} = \max_{A \subset S_k, |A|=h} \min_{ i \in B \subset S_k, |B|=h} \hat{\mathcal{D}}_k(A,B)$ where 
$$\hat{\mathcal{D}}_k(A,B) = \frac{\sum_{(j,\ell) \in B^2 \backslash (A \cap B)^2, j \neq \ell} (1-\hat{\sigma}_{j,\ell, n_k}) }{ \sum_{(j,\ell) \in A^2 \backslash (A \cap B)^2, j \neq \ell} (1-\hat{\sigma}_{j,\ell, n_k})}  $$
with the notation $\frac{0}{0}=1$.\\
(3) Let $i_k = \arg\max_{i \in S_k} U_{i,k}$ and $S_{k+1} = S_k \backslash \{i_k\}$ . \\
\textbf{Output:} $ \hat{S} = S_{K-h+1}$.
} 
\caption{SR-C (Successive Rejects for Correlation) algorithm }
\end{figure}

\begin{theorem}
The probability of error of SR-C satisfies 
$$ \P_{\Sigma} (\hat{S} \neq S^{*}) \leq  K^3 \exp\left(- \frac{c(n-K-1)  }{\overline{\log(\frac{K}{h})} \mathrm{H_C} }  \right) $$ 
where $c$ is an universal constant.
\end{theorem}
\begin{proof} Without loss of generality, we assume that $S^{*}=[h]$ and $$1=R_{1,\Sigma}=\ldots = R_{h,\Sigma} < R_{h+1,\Sigma} \leq \ldots \leq R_{K,\Sigma}. $$ 
For $k=1,\ldots,K-h$, define the event 
$$G_k = \bigcap_{j,\ell \in [K], j \neq \ell} \left\{ R_{K+1-k,\Sigma}^{-\frac{1}{4}}(1-\sigma_{j,\ell})  <   1-\hat{\sigma}_{j,\ell, n_k} < R_{K+1-k,\Sigma}^{\frac{1}{4}}(1-\sigma_{j,\ell})   \right\} $$
and $G=\cap_{k \in [K-h]} G_k$. $G$ is the event when the correlation estimators are close to the true correlations. Observe that when $G$ occurs, the following inequalities hold true 
$$  R_{K+1-k, \Sigma}^{-\frac{1}{2}} \mathcal{D}_{\Sigma}(A,B)  < \hat{\mathcal{D}}_k (A,B)  < R_{K+1-k, \Sigma}^{\frac{1}{2}} \mathcal{D}_{\Sigma}(A,B) .$$ The rest of the proof is organized into two steps. \\

\noindent\textbf{Step One:} In this step, we prove that when $G$ occurs, SR-C always finds the optimal subset, i.e. $S_{K-h+1} = S^{*}$.  We do this by induction. Fix a $k \in [K-h]$ and assume that none of the optimal arms have ever been rejected before the round $k$. We need to show that none of them are rejected at the end of the round $k$. 

First, there must exist a suboptimal arm $\underline{i}$ that belongs to $\{K+1-k,\ldots, K\} \cap S_k$ since $|S_k|= K+1-k$. Then for any $B$ such that $\underline{i} \in B \subset S_k, |B|=h$, one has
$$  \hat{\mathcal{D}}_k (S^{*},B) > R_{K+1-k,\Sigma}^{-\frac{1}{2}} \mathcal{D}_{\Sigma}(S^{*},B) \geq R_{K+1-k, \Sigma}^{\frac{1}{2}} $$
where the last inequality follows from  $\mathcal{D}_{\Sigma}(S^{*},B) \geq R_{\underline{i}, \Sigma} \geq R_{K+1-k, \Sigma}$. Therefore, we have $U_{\underline{i},k} > R_{K+1-k, \Sigma}^{\frac{1}{2}} $. 

Next, for any $A$ such that $A \subset S_k, |A|=h$, one has
$$  \hat{\mathcal{D}}_k (A, S^{*}) < R_{K+1-k,\Sigma}^{\frac{1}{2}} \mathcal{D}_{\Sigma}(A,S^{*}) \leq R_{K+1-k, \Sigma}^{\frac{1}{2}}. $$ 
This shows that for any optimal arm $\overline{i} \in S^{*}$, we have $U_{\overline{i},k} < R_{K+1-k, \Sigma}^{\frac{1}{2}} < U_{\underline{i},k} $. Thus, $\underline{i}$ is not removed at the end of the round $k$. \\ 

\noindent\textbf{Step Two:} From step one, we know that the probability of error of SR-C is bounded by $\P(G^{c})$. Using Lemma 5, one obtains 

\begin{eqnarray*}
\P(G^{c}) & \leq &  K(K-1) \sum_{k=1}^{K-h} \exp\left(- n_k \cdot \alpha\left( R_{K+1-k,\Sigma}^{\frac{1}{4}} \right) \right)\\
  & \leq &  K(K-1) \sum_{k=1}^{K-h} \exp\left(- \frac{c_1(\frac{1}{4})(n-K-1) \alpha\left( R_{K+1-k,\Sigma} \right) }{\overline{\log(\frac{K}{h})}(K+1-k)}  \right) \\
  & = &  K(K-1) \sum_{k=h+1}^{K} \exp\left(- \frac{c_1(\frac{1}{4})(n-K-1) \alpha\left( R_{k,\Sigma} \right) }{\overline{\log(\frac{K}{h})}k}  \right) \\
  & \leq &  K(K-1) \sum_{k=h+1}^{K} \exp\left(- \frac{c_1(\frac{1}{4})(n-K-1)  }{\overline{\log(\frac{K}{h})} \mathrm{H_C} }  \right) \\
  & \leq &  K^3 \exp\left(- \frac{c_1(\frac{1}{4})(n-K-1)  }{\overline{\log(\frac{K}{h})} \mathrm{H_C} }  \right),
\end{eqnarray*} 
which completes the proof.
\end{proof}

\section{Fixed-Confidence Setting: Successive Elimination for Correlation }
In this section, we introduce and study a new algorithm, SE-C (Successive Elimination for Correlation). See Figure 3 for its complete description.  SE-C is a more dynamic algorithm than SR-C. SE-C updates the suboptimality estimates for the existing arms at each time step rather than at the end of each round. And while SR-C always rejects an arm at each of some prespecified time steps, SE-C eliminates arms only when those arms are judged suboptimal with enough confidence.  The following theorem shows that SE-C needs at most $\tilde{\mathcal{O}}\left( \mathrm{H_C} +K \right)$ samples to find the optimal subset.

\begin{figure}[h!]
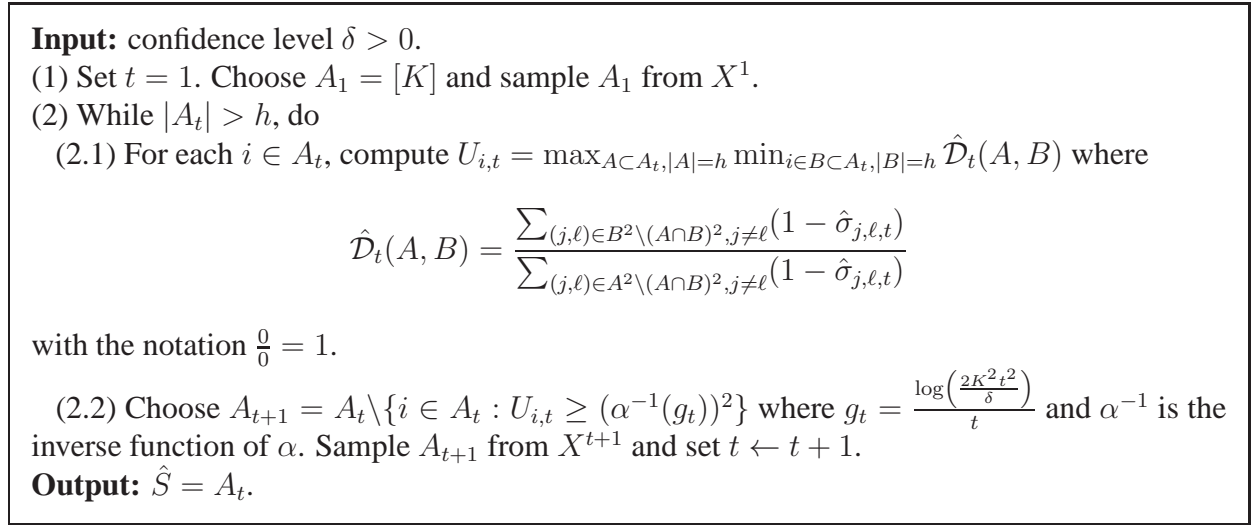

\bookbox{ 
\textbf{Input:} confidence level $\delta > 0$.  \\
(1) Set $t=1$. Choose $A_1 = [K]$ and sample $A_1$ from $X^1$. \\
(2) While $|A_t| > h$, do  \\
 \text{  }\text{  } (2.1) For each $i \in A_t$, compute $U_{i,t} = \max_{A \subset A_t, |A|=h} \min_{ i \in B \subset A_t, |B|=h} \hat{\mathcal{D}}_t(A,B)$ where 
$$\hat{\mathcal{D}}_t(A,B) = \frac{\sum_{(j,\ell) \in B^2 \backslash (A \cap B)^2, j \neq \ell} (1-\hat{\sigma}_{j,\ell, t}) }{ \sum_{(j,\ell) \in A^2 \backslash (A \cap B)^2, j \neq \ell} (1-\hat{\sigma}_{j,\ell, t})}  $$
with the notation $\frac{0}{0}=1$.\\
 \text{  }\text{  }  (2.2) Choose $A_{t+1} = A_t \backslash \{i \in A_t : U_{i,t} \geq (\alpha^{-1}(g_t))^2 \}$ where  $g_t= \frac{\log \left(  \frac{2K^2t^2}{\delta} \right) }{t}$ and $\alpha^{-1}$ is the inverse function of  $\alpha$. Sample $A_{t+1}$ from $X^{t+1}$ and set $t \leftarrow t+1$.  \\
\textbf{Output:} $ \hat{S} = A_t$.
} 
\caption{ SE-C (Successive Elimination for Correlation) algorithm }
\end{figure}

\begin{theorem}
With probability at least $1-\delta$, SE-C returns the optimal subset $S^{*}$ and the number of samples it uses is bounded by $$ \mathcal{O} \left( h \max\left( 1, \frac{\log \left( \frac{2K^2}{\delta \alpha(R_{(h+1),\Sigma})}  \right)}{\alpha(R_{(h+1),\Sigma})} \right)  +   \sum_{i=h+1}^K \max\left( 1, \frac{\log \left( \frac{2K^2}{\delta \alpha(R_{(i),\Sigma})}  \right)}{\alpha(R_{(i),\Sigma})} \right) \right) = \tilde{\mathcal{O}}(\mathrm{H_C} + K) .$$
\end{theorem}
\begin{proof} Without loss of generality, we assume that $S^{*}=[h]$ and $$1=R_{1,\Sigma}=\ldots = R_{h,\Sigma} < R_{h+1,\Sigma} \leq \ldots \leq R_{K,\Sigma}. $$ 
Define the event 
$$G = \bigcap_{j,\ell \in [K], j \neq \ell} \left\{ \forall t \geq 1,  (\alpha^{-1}(g_t))^{-1}(1-\sigma_{j,\ell})  <   1-\hat{\sigma}_{j,\ell, t} < \alpha^{-1}(g_t) (1-\sigma_{j,\ell})   \right\}. $$
Observe that when $G$ occurs, the following inequalities hold true 
$$   (\alpha^{-1}(g_t))^{-2} \mathcal{D}_{\Sigma}(A,B)  < \hat{\mathcal{D}}_t (A,B)  < (\alpha^{-1}(g_t))^{2}  \mathcal{D}_{\Sigma}(A,B) .$$ 
Once again, we use Lemma 5 to bound the probability of $G^c$ as 
$$ \P(G^c) \leq K(K-1)\sum_{t=1}^{+\infty}\exp(-t g_t ) \leq  \sum_{t=1}^{+\infty} \frac{\delta}{2t^2} \leq \delta. $$
Now to prove the theorem, it is enough to show that under the event $G$, SE-C always terminates with the optimal subset $S^{*}$ and the number of samples it uses is bounded as claimed. These two facts are proved separately in the following two steps.  \\

\noindent\textbf{Step One:}
 Fix a $t \geq 1$. Assume that all of the optimal arms are still present at the beginning of time $t$, that is, $S^{*} \in A_t$. 
For any $A$ such that $A \subset A_t, |A|=h$, one has
$$  \hat{\mathcal{D}}_t (A,S^{*})  < (\alpha^{-1}(g_t))^{2}  \mathcal{D}_{\Sigma}(A,S^{*}) \leq (\alpha^{-1}(g_t))^{2} .$$ 
This implies that for any optimal arm $\overline{i} \in S^{*}$, we have $U_{\overline{i},t} < (\alpha^{-1}(g_t))^{2}  $ and thus $\underline{i}$ is not eliminated at the end of round $t$. \\

\noindent\textbf{Step Two:} Consider a suboptimal arm $\underline{i}$. If $\underline{i}$ is not eliminated before or at the end of time $t$, then 
\begin{eqnarray*}
(\alpha^{-1}(g_t))^{2}   & > &  U_{\underline{i},t} \\
 &\geq &   \min_{ \underline{i} \in B \subset A_t, |B|=h} \hat{\mathcal{D}}_t(S^{*},B) \\
&\geq &  (\alpha^{-1}(g_t))^{-2}  \min_{ \underline{i} \in B \subset A_t, |B|=h} \mathcal{D}_{\Sigma}(S^{*},B) \\
&= &  (\alpha^{-1}(g_t))^{-2} R_{\underline{i}, \Sigma}
\end{eqnarray*}
Therefore, $\underline{i}$ cannot be sampled more than $t$ times as long as $t$ is such that $g_t \leq \alpha\left(  R_{\underline{i}, \Sigma}^{\frac{1}{4}} \right)$. The latter inequality holds true as long as  $g_t \leq c_2\left(\frac{1}{4}\right) \alpha\left(   R_{\underline{i}, \Sigma}\right)$.   Thus, any suboptimal arm $i$ cannot be sampled more than 
$ \mathcal{O} \left( \max\left( 1, \frac{\log \left( \frac{2K^2}{\delta \alpha(R_{i,\Sigma})}  \right)}{\alpha(R_{i,\Sigma})} \right) \right) $ times. Moreover, the algorithm stops once all the supoptimal arms have been eliminated. Thus, an optimal arm cannot be sampled more than $ \mathcal{O} \left( \max\left( 1, \frac{\log \left( \frac{2K^2}{\delta \alpha(R_{h+1,\Sigma})}  \right)}{\alpha(R_{h+1,\Sigma})} \right) \right) $ times, which completes the proof.
\end{proof}

\section{Discussion}
This work is a first step towards understanding the hardness of finding the most correlated arms with an adaptive sampling scheme. We proposed two algorithms SR-C and SE-C, and we show that both succeed with at most $\tilde{ \mathcal{O}}\left(h\cdot \alpha (R_{(h+1),\Sigma})^{-1}+ \sum_{i=h+1}^{K} \alpha (R_{(i),\Sigma})^{-1}+K \right)$ samples. 
The result of Section \ref{sec:nonadaptivelowerbound} together with known arguments from the best arm identification literature strongly indicates that the term
$\sum_{i=h+1}^{K} \alpha (R_{(i),\Sigma}) ^{-1}$ 
is unavoidable.
On the other hand it is clear that the term $h\cdot \alpha (R_{(h+1),\Sigma})^{-1}$ 
is suboptimal in some cases.
For example, consider the problem instances that are equivalent, up to a permutation of the arms, to the one with the correlation matrix $\Sigma$ described in Section \ref{sec:nonadaptivelowerbound}. In these cases there is always a block of $h-1$ perfectly correlated arms, and thus 
one vector sample will allow to identify this block.
Consequently, for these problem instances a trivial modification of SR-C and SE-C will result in a reduced upper bound $\tilde{ \mathcal{O}}\left(\alpha (R_{(h+1),\Sigma})^{-1}+ \sum_{i=h+1}^{K} \alpha (R_{(i),\Sigma})^{-1}+K \right)$.  It is an interesting challenge to design algorithms that can significantly reduce the term $\tilde{ \mathcal{O}} \left( h\cdot \alpha (R_{(h+1),\Sigma})^{-1} \right)$ for general problem instances. It is worth noting that a similar term $\tilde{ \mathcal{O}} \left( h\cdot \Delta_{(h+1)}^{-2} \right)$ was observed when the best arm identification algorithm Successive Rejects was used for the task of identifying the $h$ best arms. This term can be significantly reduced by allowing the algorithm to accept seemingly optimal arms early, in the same way as seemingly suboptimal arms are rejected early, see \cite{KTAS12}  and \cite{BWV13} for details. Unfortunately, the same trick can not be applied easily to the problem of most correlated arms identification. The main difficulty is that when an optimal arm is accepted early, the remaining optimal arms are not necessarily the most mutually correlated arms among all the remaining arms, thus making the identification of the remaining optimal arms difficult (if not impossible). To summarize, a novel algorithmic idea is needed to improve our upper bound. Another interesting direction of further work is to prove a lower bound on the number of samples that any adaptive strategy must use.   \\


\bibliographystyle{plainnat}
\bibliography{newbib}

\newpage
\appendix

\section{Technical Lemmas}

\begin{lemma}{(Concentration Inequalities for Chi-Square Distributions)}
Let $Y$ be a random variable following the chi-square distribution with degree of freedom $t\in \N$.  Then for any $\theta \geq 1$, the following concentration inequalities hold:
$$\P\left(\frac{Y}{t} \leq \frac{1}{\theta} \right) \leq \exp\left(-t\cdot \alpha(\theta)\right) $$
$$\P\left(\frac{Y}{t} \geq \theta \right)  \leq \exp\left(-\frac{t}{2}(\theta-1-\log \theta)\right) \leq \exp\left(-t\cdot \alpha(\theta)\right).$$
\end{lemma}
\begin{proof}
The proof is a simple application of Chernoff's bounding technique.
\end{proof}

\begin{lemma}
Let $\mu_0$ and $\mu_1$ be two probability distributions on some set $\mathcal{X}$ , with $\mu_0$ absolutely continuous with respect to  $\mu_1$. Let $X$ be a random variable taking values in  $\mathcal{X}$. Then for any measurable function $\phi : \mathcal{X} \rightarrow \{0,1\}$, one has
$$ \max\left(\P_{X \sim \mu_0}(\phi(X)=1), \P_{X \sim \mu_1}(\phi(X)=0)\right) \geq \frac{1}{4}\exp\left(-\KL(\mu_0,\mu_1)\right).$$
\end{lemma}
\begin{proof}
See Chapter 2 in \cite{T09} for a proof.
\end{proof}

\begin{lemma} Let $1 > \rho_0 > \rho_1 \geq 0$ and $R=\frac{1-\rho_1}{1-\rho_0}$. Let $\Sigma_0 = \bigl( \begin{smallmatrix} 1 & \rho_0 \\ \rho_0 & 1 \end{smallmatrix} \bigr)$ and $\Sigma_1 = \bigl( \begin{smallmatrix} 1 & \rho_1 \\ \rho_1 & 1 \end{smallmatrix} \bigr)$. Then 
$$ \alpha(R) \leq \KL\left( \mathcal{N}(0, \Sigma_0), \mathcal{N}(0, \Sigma_1) \right)
=  \KL\left( \mathcal{N}(\rho_0, 1-\rho_0^2), \mathcal{N}(\rho_1, 1-\rho_1^2) \right)  \leq c\cdot \alpha(R) .$$
where $c > 1$ is an universal constant.
\end{lemma}
\begin{proof} 
Recall the following general formula for the $\KL$-divergence between two Gaussian distributions on $\R^k$,
$$ \KL\left( \mathcal{N}(\mu_0, \Sigma_0), \mathcal{N}(\mu_1, \Sigma_1) \right) =
  \frac{1}{2} \left( \log \left( \frac{\det \Sigma_1}{\det \Sigma_0} \right) + Tr(\Sigma_1^{-1}\Sigma_0) -k + (\mu_1-\mu_0)^T \Sigma_1^{-1} (\mu_1-\mu_0) \right).$$ 
By straightforward application of the above formula, one has 
\begin{eqnarray*}
\KL\left( \mathcal{N}(0, \Sigma_0), \mathcal{N}(0, \Sigma_1) \right)
  & = & \KL\left( \mathcal{N}(\rho_0, 1-\rho_0^2), \mathcal{N}(\rho_1, 1-\rho_1^2) \right) \\ 
  & = & \frac{1}{2} \left(  \log\frac{1-\rho_1}{1-\rho_0} + \frac{1-\rho_0}{1-\rho_1} - 1 + \log\frac{1+\rho_1}{1+\rho_0} + \frac{1+\rho_0}{1+\rho_1} - 1  \right) \\ 
  & = &   \alpha (R) +  \beta \left( \frac{1+\rho_0}{1+\rho_1} \right)  \\ 
\end{eqnarray*}   
where the function $\beta$ is defined as $\beta(\theta) = \theta-1-\log \theta $, for  $\theta \geq 1$. It is easy to see that $\beta$ is a positive, strictly increasing function on $[1,+\infty [$. Since $ \frac{1+\rho_0}{1+\rho_1}  \leq \min(2,R)$, one has 
$$ \beta \left( \frac{1+\rho_0}{1+\rho_1} \right) \leq \beta( \min(2,R) ) .$$ Moreover, because $\beta(R) = \Theta\left((R-1)^2\right)$ when $R \rightarrow 1$, there exists a universal constant $c > 0$ such that $ \beta( \min(2,R) ) \leq c \cdot \alpha(R)  $ for any $R \geq 1$, which completes the proof.
\end{proof}

\begin{lemma}
Consider the two correlation matrices $\Sigma$, $\Sigma^{'}$ described in section 3.2. One has
$$ \KL\left( \mathcal{N}(0, \Sigma^{'}), \mathcal{N}(0, \Sigma)\right) \leq  c \cdot \alpha(R_{h+1, \Sigma}) $$
where $c$ is an universal constant.
\end{lemma}
\begin{proof}
Let $Z_0, Z_h, Z_{h+1}, \ldots, Z_{K}$ be i.i.d. standard Gaussian random variables. Define two random vectors $W=(W_1,...,W_K)$  and $W^{'}=(W^{'}_1,...,W^{'}_K)$ as
$$ W_i = 
\begin{cases} 
Z_0,  & \mbox{if } i \leq h-1 \\
\rho_i Z_0 + \sqrt{1-\rho_i^2} Z_i , & \mbox{if } i \geq h 
\end{cases} 
$$
and
$$ W^{'}_i = 
\begin{cases} 
Z_0,  & \mbox{if } i \leq h-1 \\
\rho_i Z_0 + \sqrt{1-\rho_i^2} Z_i , & \mbox{if } i \geq h \mbox{ and } i \neq h+1  \\
\rho_{h+1}^{'} Z_0 + \sqrt{1-\rho^{'2}_{h+1}} Z_{h+1} , & \mbox{if } i = h+1
\end{cases} .$$  
Then it is easy to see that $W$ and $W^{'}$ are two centered Gaussian vectors with covariance matrices $\Sigma$ and $\Sigma^{'}$. Moreover, conditioning on $W_1 = z$ for some $z \in \R$, $W_2,\ldots, W_K$ are independent and Gaussian distributed as $W_i \sim \mathcal{N}(z,0)$ for $ 2 \leq i \leq h-1$ and $W_i \sim \mathcal{N}(\rho_i z,1-\rho_i^2)$ for $ i \geq h$. Similarly, conditioning on $W^{'}_1 = z$, $W^{'}_2,\ldots, W^{'}_K$ are independent and Gaussian distributed as $W^{'}_i \sim \mathcal{N}(z,0)$ for $ 2 \leq i \leq h-1$, $W^{'}_{h+1} \sim \mathcal{N}(\rho^{'}_{h+1} z,1-\rho_{h+1}^{'2})$ and  $W^{'}_i \sim \mathcal{N}(\rho_i z,1-\rho_i^2)$ for $ i \geq h$ and $i \neq h+1$. In what follows, we use $\mathcal{L}(Y)$ to represent the distribution of a random variable $Y$. Using the conditioning formula for $\KL$-divergence, one has  
\begin{eqnarray*}
&& \KL\left( \mathcal{N}(0, \Sigma^{'}), \mathcal{N}(0, \Sigma)\right) \\
 &= & \KL\left( \mathcal{L}(W_1^{'}), \mathcal{L}(W_1) \right)  +
\E_{z \sim W^{'}_1} \left[ \KL\left( \mathcal{L}((W^{'}_2,...,W^{'}_K) | W^{'}_1 = z), \mathcal{L}((W_2,...,W_K) | W_1 = z) \right) \right] \\
 &= & 
\E_{z \sim \mathcal{N}(0,1)} \left[ \KL\left( \mathcal{N}(\rho^{'}_{h+1} z,1-\rho_{h+1}^{'2}), \mathcal{N}(\rho_{h+1} z,1-\rho_{h+1}^2) \right)  \right] \\
& = & \E_{z \sim \mathcal{N}(0,1)} \left[ \frac{1}{2} \left( \log \left( \frac{1-\rho_{h+1}^2}{1-\rho_{h+1}^{'2}} \right) + \frac{1-\rho_{h+1}^{'2}}{1-\rho_{h+1}^2}  - 1 + \frac{(\rho_{h+1}-\rho_{h+1}^{'})^2 z^2 }{1-\rho_{h+1}^2 }  \right)  \right] \\
& = & \frac{1}{2} \left( \log \left( \frac{1-\rho_{h+1}^2}{1-\rho_{h+1}^{'2}} \right) + \frac{1-\rho_{h+1}^{'2}}{1-\rho_{h+1}^2}  - 1 + \frac{(\rho_{h+1}-\rho_{h+1}^{'})^2  }{1-\rho_{h+1}^2 }  \right) \\
& = &  \KL\left( \mathcal{N}(\rho^{'}_{h+1} ,1-\rho_{h+1}^{'2}), \mathcal{N}(\rho_{h+1} ,1-\rho_{h+1}^2) \right)   \\
& \leq & c \cdot \alpha\left( \frac{1-\rho_{h+1}}{1-\rho^{'}_{h+1}} \right) \\
\end{eqnarray*}
where $c$ is an universal constant and the last step follows from Lemma 7. To conclude, it is enough to observe that $\alpha\left( \frac{1-\rho_{h+1}}{1-\rho^{'}_{h+1}} \right) = \alpha\left( \left(\frac{1-\rho_{h+1}}{1-\rho_{h}} \right)^2 \right) =\alpha(R_{h+1, \Sigma}^2) \leq c_2(2) \alpha(R_{h+1, \Sigma}) $.

\end{proof}

\end{document}

%% file: Commands.tex
\renewcommand{\phi}{\varphi}

\renewcommand{\P}{\mathbb{P}}
\newcommand{\E}{\mathbb{E}}
\newcommand{\N}{\mathbb{N}}
\newcommand{\R}{\mathbb{R}}

\newcommand{\KL}{\mathrm{KL}}

\def\ds1{\mathds{1}}
\renewcommand{\epsilon}{\varepsilon}

\newcommand{\argmax}{\mathop{\mathrm{argmax}}}

\renewcommand{\tilde}{\widetilde}

\newlength{\minipagewidth}
\newcommand{\bookbox}[1]{
\par\medskip\noindent
\framebox[\textwidth]{
\begin{minipage}{\minipagewidth}
{#1}
\end{minipage} } \par\medskip }

\newcommand{\beq}{\begin{equation}}
\newcommand{\eeq}{\end{equation}}

\newcommand{\beqa}{\begin{eqnarray}}
\newcommand{\eeqa}{\end{eqnarray}}

\newcommand{\beqan}{\begin{eqnarray*}}
\newcommand{\eeqan}{\end{eqnarray*}}

\def\ba#1\ea{\begin{align*}#1\end{align*}} 
\def\banum#1\eanum{\begin{align}#1\end{align}} 